\documentclass{article}
\PassOptionsToPackage{sort}{natbib}
\usepackage{geometry}
\usepackage{fullpage}

\usepackage[numbers]{natbib}

\PassOptionsToPackage{hyphens}{url}\usepackage{hyperref}
\usepackage{booktabs}       
\usepackage{amsfonts}       
\usepackage{nicefrac}       
\usepackage{microtype}      
\usepackage{wrapfig,lipsum,booktabs}
\usepackage{multicol}
\usepackage{enumitem}

\usepackage{graphicx} 
\usepackage{subcaption}
\usepackage{algorithm}
\usepackage{algorithmic}
\usepackage{bm}
\usepackage[dvipsnames]{xcolor}
\usepackage[skins]{tcolorbox}
\usepackage{mathdots}
\usepackage{float}
\usepackage{amsmath, amssymb,amsthm}
\usepackage{xspace}
\usepackage{array}
\usepackage{amsthm}
\usepackage{lmodern}

\newcommand{\paraf}[1]{\noindent\textbf{#1}}

\newcolumntype{C}[1]{>{\centering\arraybackslash}m{#1}}
\usepackage{tabu}
\usepackage{enumitem}

\theoremstyle{plain}
\newtheorem{theorem}{Theorem}
\newtheorem{lemma}{Lemma}

\newtheorem{assumption}{Assumption}

\theoremstyle{definition}
\newtheorem{definition}[theorem]{Definition}

\newcommand{\M}{{\mathcal{M}}}

\newcommand{\name}{\texttt{DiffSketch}\xspace}

\newcommand{\fedavg}{\texttt{FedAvg}\xspace}
\newcommand{\cpsgd}{\texttt{cpSGD}\xspace}

\begin{document}

\title{Privacy for Free: Communication-Efficient Learning with Differential Privacy Using Sketches}
\author{\begin{tabular}{cccc}
\small Tian Li & \small Zaoxing Liu & \small Vyas Sekar & \small Virginia Smith\tabularnewline 
\small CMU & \small  CMU  & \small  CMU & \small  CMU \tabularnewline
\small{\texttt{\href{mailto:tianli@cmu.edu}{tianli@cmu.edu}}} & \small{\texttt{\href{mailto:zaoxing@cmu.edu}{zaoxing@cmu.edu}}} &
\small{\texttt{\href{mailto:vsekar@andrew.cmu.edu}{vsekar@andrew.cmu.edu}}} & \small{\texttt{\href{mailto:smithv@cmu.edu}{smithv@cmu.edu}}}
\end{tabular}}
\date{}
\maketitle

\begin{abstract}

Communication and privacy are two critical concerns in distributed learning. 
Many existing works treat these concerns separately. In this work, we argue that a natural connection exists between methods for communication reduction and privacy preservation in the context of distributed machine learning. In particular, we prove that Count Sketch, a simple method for data stream summarization, has inherent differential privacy properties.\footnote{\textit{We note that the definition of local differential privacy used in our work is weaker than the standard local privacy definition. In addition, we are aware of some issues with our current proof of the differential privacy properties of Count Sketch. We are currently working on a revision of this draft.}} Using these derived privacy guarantees, we  propose a novel sketch-based framework (\name) for distributed learning, where we compress the transmitted messages via sketches to \textit{simultaneously} achieve communication efficiency and provable privacy benefits. 
Our evaluation demonstrates that \name can provide strong differential privacy guarantees (e.g., $\varepsilon=1$) and reduce communication by 20-50$\times$ with only marginal decreases in accuracy.
Compared to baselines that treat privacy and communication separately, \name improves absolute test accuracy by 5\%-50\% while offering the same privacy guarantees and communication compression.
\end{abstract}

\section{Introduction}
Communication costs and privacy concerns are two critical challenges in performing distributed machine learning with sensitive information. Indeed, these challenges are particularly relevant for the increasingly common problem of \textit{federated learning}, in which data is collected and stored across a network of devices such as mobile phones or wearable devices~\citep{mcmahan2016FedAvg,li2019federated}. Communicating data across such a distributed network in order to learn an aggregate model is both costly and can potentially reveal sensitive user information~\citep{mcmahan2018diff, carlini2018secret}. 

Many prior  efforts have separately considered communication or privacy in distributed machine learning. For example, common strategies to reduce the size of messages sent across the network include sparsification, quantization, or subsampling~\citep{konevcny2016federated}. Previous privacy-preserving methods typically add noise to guarantee differential privacy~\citep[][]{mcmahan2018diff}, or use cryptographic protocols such as secure multi-party computation (SMC)~\citep{bonawitz2017practical}. While these  approaches are effective at either reducing communication or protecting privacy independently, the two goals are largely treated as orthogonal to one another. As we demonstrate, this can be problematic when aiming to achieve both private and efficient learning, as using a direct combination of the above techniques~\citep[e.g.,][]{agarwal2018cpsgd} can significantly degrade overall accuracy (\S\ref{sec:eval:compare}).

In this work, we argue that the goals of communication reduction and privacy protection are closely related, and show that leveraging this relation has significant benefits for distributed learning. Intuitively, both tasks share an aim to reduce, mask, or transform information sent across the network in a way that preserves the underlying learning task. While  this connection has been observed in  other  settings~\citep{zhou2009differential,xiong2016randomized}, we are not aware of  other work that  formalizes these  connections and optimizes both  goals in the context of distributed machine learning. 

In particular, we identify \emph{sketching algorithms} (sketches) as a natural tool to jointly consider communication and privacy in distributed learning.
Sketches use independent hash functions to compress the input data with bounded errors, and have well-studied trade-offs between space and accuracy~\citep{ams,CountSketch,CMSketch}. While sketches have previously been used to reduce communication in distributed networks~\citep[][]{ivkin2019communication}, we show that by randomizing the data with independent hash functions, sketches can in fact also provide \textit{privacy for free}, as they  have provable differential privacy benefits (\S\ref{sec:analysis}).  More specifically, we demonstrate that sketches can be used to achieve \emph{local differential privacy}~\citep{warner1965randomized,duchi2013local} in distributed machine learning---a setting where we do not assume a trusted central server and the central server cannot see individual model updates. 

Based on our derived privacy guarantees, we introduce 
\name, a framework built upon canonical sketches such as  Count Sketch~\citep{CountSketch} to jointly optimize communication and privacy in distributed learning. While we explore \name in the context of distributed machine learning, we note that it is general in that it could potentially be applied to any task that requires estimating the mean of a set of distributed vectors.
Empirically, we verify the communication reduction  and privacy benefits of \name for applications with both distributed SGD and \fedavg, a popular method for federated learning~\citep{mcmahan2016FedAvg}. Our results demonstrate that by considering communication and privacy jointly, \name can deliver 5\%-50\%  improvements in terms of absolute test accuracy compared with baselines for fixed levels of privacy and communication.

{\bf Contributions\footnote{{We note that a preliminary version of our vision appeared in~\citep{liu2019Enhancing}, where we suggested the promise of using sketching to enhance privacy in federated learning. However, in this preliminary work we did not prove any differential privacy guarantees of sketches, propose a distributed learning framework to capitalize on the connection, or explore the privacy benefits empirically.}}.} 
In this work, we provide the first differential privacy guarantees of Count Sketch without additional mechanisms (\S\ref{sec:analysis}). Based on these derived privacy guarantees, we design \name, a communication-efficient and differentially-private distributed learning framework using Count Sketch as a building block (\S\ref{sec:algorithms}). 
Our evaluation demonstrates that \name offers strong privacy guarantees and communication reduction benefits with significantly improved accuracy compared to baselines (\S\ref{sec:eval}).

\section{Related Work and Background}\label{sec:related_work}

\textbf{Communication-efficient learning.} 
Communication is a key bottleneck in distributed learning, and has been studied extensively in classical distributed (e.g., data center) computing as well as emerging  applications of federated learning. A common approach to reducing communication is to limit the size of messages sent across the network using various compression methods~\citep[e.g.,][]{konevcny2016federated,caldas2018expanding,ivkin2019communication}. This is particularly useful for applications such as deep learning, where the messages sent are  similar in size to the model and can thus be quite large. As we show in this work, compression methods such as sketching are related to the goals of privacy preservation, but are typically explored in isolation. 

A second, orthogonal approach to communication-efficiency is to reduce the total number of communication rounds in training by developing flexible and efficient distributed optimization methods~\citep{AIDE_reddi_16,COCOA_Smith_2016,fed_multitask_smith_2017,mcmahan2016FedAvg,local_SGD_stich_18}. 
While not the focus of this work, we explore our framework in conjunction with one such canonical communication-efficient optimization method, \fedavg~\citep{mcmahan2016FedAvg}, in \S\ref{sec:algorithms}.

\paragraph{Privacy in distributed learning.}\label{sec:related_work:privacy} 

In a variety of distributed learning scenarios handling user-generated data, privacy has become increasingly important to consider. Indeed, privacy concerns can even motivate the need for distributed learning---for example, privacy is a key motivation in federated learning, as it  necessitates that raw user data remain local and thus distributed across the network~\citep{mcmahan2016FedAvg}.  There are two common approaches to preserve privacy in distributed settings, based either on statistical or cryptographical foundations.
The first is to provide differential privacy 
by adding random perturbations (e.g., random noise drawn from a Gaussian or Laplacian distribution) to the output of a function~\citep{mcmahan2018diff,li2019differentially,agarwal2018cpsgd,abadi2016deep}. In distributed contexts, it is common to think of differential privacy as either being enforced in a ``local'' or ``global'' sense depending on whether the server is a trusted party. We provide more formal definitions of differential privacy in \S\ref{sec:analysis:preli}, and demonstrate that our framework, \name, is able to achieve strong local differential privacy guarantees via sketches. 

Another line of work is based on cryptography, including secure multiparty computation (SMC)~\citep[][]{bonawitz2017practical,ghazi2019scalable,chen2019secure}. SMC protects users' privacy by encrypting individual inputs such that multiple parties can jointly compute a function without learning about the input information from any party. However, cryptographic-based approaches are not well-suited to distributed learning, as they can incur significant communication and computation overheads~\citep{bassily2017practical}. In general, these types of  privacy-preserving methods  tend to increase the  communication cost in distributed settings.

\paragraph{Connections between communication and privacy.}\label{sec:related_work:connection} 
Some recent works have explored connections between communication compression  and privacy~\citep{xiong2016randomized,zhou2009differential}. However, they focus on much simpler settings such as multiplicative database transformation~\citep{zhou2009differential}, or transmitting raw data in sensor networks~\citep{xiong2016randomized}, and can only preserve limited information such as the covariance of the original data~\citep{zhou2009differential}. \citet{duchi2019lower} explore an equivalence between estimation under privacy constraints and estimation under communication constraints, but for different purposes of developing lower bounds to showcase the fundamental limitations of differential privacy. We explore such connections in distributed learning both theoretically and empirically through sketching, and demonstrate that \name achieves competitive accuracy with guaranteed differential privacy and convergence (\S\ref{sec:eval}). 
Finally, recent work  called \cpsgd~\citep{agarwal2018cpsgd} proposes modifications to distributed SGD to make the method both private and communication-efficient. However, the authors treat these as separate issues, and develop different approaches to address each within \cpsgd. In \S\ref{sec:eval} we compare directly with \cpsgd and demonstrate that by \emph{separately} reducing communication and enforcing privacy, errors in \cpsgd are compounded; by instead considering these {\em jointly}, \name provides substantially higher accuracy.

\paragraph{Sketches.}
Sketching algorithms (sketches) provide   approximate estimates of   different statistics (such as computing the distinct count or average) over a dataset, and have been widely used in streaming data processing~\citep{ams,SpaceSavings,cormode2008finding}, databases~\citep{CountSketch,CMSketch,cormode2011sketch}, and network measurement~\citep{OpenSketch,univmon,nitrosketch}.  
Recently sketches have also been used to improve large-scale machine learning, e.g., by compressing model updates~\citep{jiang2018sketchml}, identifying significant coordinates in the gradient vectors~\citep{ivkin2019communication}, and reducing memory usage in model training~\citep{spring2019compressing}. However, these works do not focus on  privacy. 

Some recent efforts  consider extensions to sketches to provide  differential privacy  through the use of  random noise~\citep{MelisDC16}, random sampling~\citep{zhu2019federated}, or other randomization~\citep{bassily2017practical,ghazi2019scalable}. However, they require adding these mechanisms on top of sketches, and do not investigate any inherent privacy properties of sketches themselves. 
Our analysis shows that sketches have inherent differential privacy properties (\S\ref{sec:analysis}) and can be utilized to design private distributed learning algorithms that achieve better communication and accuracy trade-offs than state-of-the-art mechanisms (\S\ref{sec:algorithms}).

\section{Differential Privacy of Sketches}\label{sec:analysis}
In this section, we analyze the {intrinsic} differential privacy properties of a canonical sketch (Count Sketch). We begin by providing relevant background on Count Sketch and differential privacy, and then present our differential privacy bounds in \S\ref{sec:analysis:dp}. 

\subsection{Preliminaries}\label{sec:analysis:preli}

Count Sketch is a common sketching algorithm that can be used to compress real-valued vectors.  
As depicted in Figure~\ref{fig:countsketch}, Count Sketch works by using $t$ independent hash functions to map each element in a vector $g \in \mathbb{R}^n$ to $t$ distinct bins in arrays of size $k \ll n$. This may cause some collisions, and the goal of the sketch is therefore to approximate the true values within bounded errors when queried.    
Due to the linearity of Count Sketch, it can be used to compress messages in distributed settings, as it is easy to merge sketched messages since $\bm{S}(g_1 + g_2) = \bm{S}(g_1) + \bm{S}(g_2)$ for a sketch method $\bm{S}$ and vectors $g_1$ and $g_2$. 

\begin{figure}
    \centering
    \includegraphics[width=0.5\textwidth]{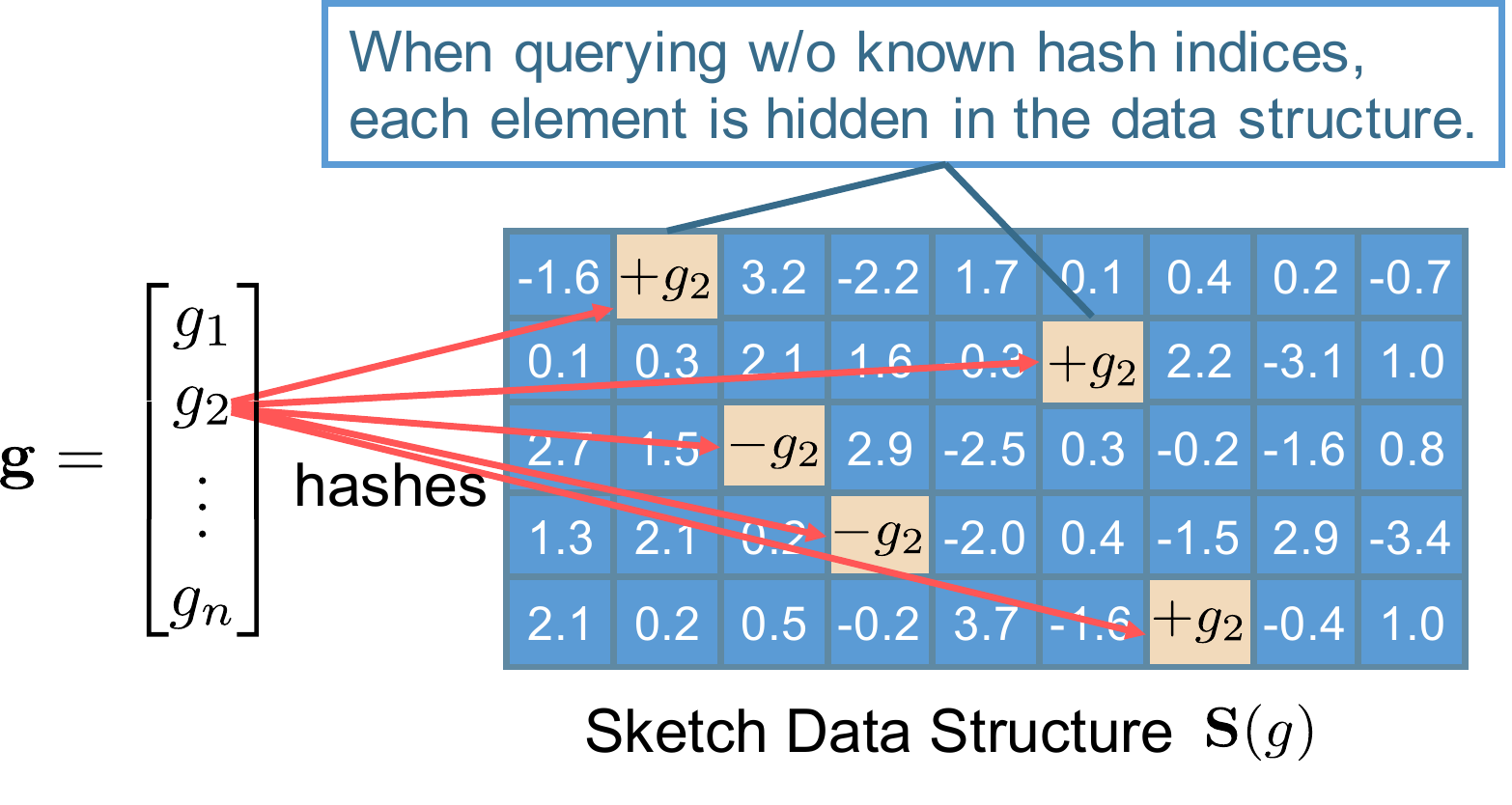}
    \caption{An illustration of using Count Sketch $\bm{S}_{5 \times 9}$ to compress a vector $g \in \mathbb{R}^n$. Each element $g_i \in g$ is mapped to five bins in five counter arrays via independent hash functions. It is difficult for third parties to infer the original inputs. We prove that the Count Sketch algorithm is differentially private in \S\ref{sec:analysis:dp}.
    }
    \label{fig:countsketch}
\end{figure}

We summarize Count Sketch in Algorithm~\ref{alg:countsketch}, and defer interested readers to~\citet{CountSketch} for additional background. The algorithm consists of two key operations for any vector $g \in \mathbb{R}^n$:
\begin{enumerate}
    \item \emph{Encoding ${g}$ into a sketch table $\bm{S}_{t \times k}({g})$.} To encode every $g_i \in g$, we use a set of pairwise independent hash functions $\{h_{j,1 \leq j \leq t}: [n] \to [k] \}$, along with a set of 2-wise independent sign hash functions $\{\text{sign}_{j,1 \leq j \leq t}: [n] \to \{+1,-1\}\}$ to map each $g_i$ to $t$ different bins in the arrays of the table.
    \item \emph{Querying $\bm{S}_{t \times k}({g})$ to obtain an estimation $\tilde{g}$ of ${g}$.} To query $g_i$ from $\bm{S}_{t \times k}({g})$, we query the median of $t$ approximated values identified by the indexes of $h_j(i)~(1\leq j \leq t)$. To achieve $\mu L_2$ additive errors on $\tilde{g}$ with $1-\delta$ probability, the size of the table ($t$ and $k$) are configurable as $t=O\left(\ln(\frac{1}{\delta}\right))$ and $k=O\left(\frac{1}{\mu^2}\right)$~\citep{CountSketch}.
\end{enumerate}

When analyzing Count Sketch, our insight is that the randomization from the independent hash functions potentially provides differential privacy in the sketch table. Intuitively, in Count Sketch, the differences in the output distribution caused by minor changes in the input should be bounded with high probability. This observation mirrors the definition of differential privacy, as formally stated below. 

\begin{definition}[\textit{$\varepsilon$-differential privacy~\citep{dwork2011differential}}]\label{def:dp}
A randomized mechanism $\M$ satisfies $\varepsilon$-differential privacy, if for input data $D_1$ and $D_2$ differing by up to one element,  and for any output $S$ of $\M$,
\begin{equation*}
    Pr[\M(D_1) \in S] \le e^{\varepsilon} \cdot Pr[\M(D_2) \in S].
\end{equation*}
\end{definition}

Informally, the above definition states that when $\varepsilon$ is small enough, the two output distributions are closer to each other, and the method $\mathcal{M}$ becomes more private as it is difficult  to determine whether the input is $D_1$ or $D_2$.
There are two common privacy models in distributed networks: \emph{local privacy} where the central aggregator/server is not trusted, and \emph{global privacy} where we assume a trusted central server~\citep{warner1965randomized,duchi2013local}. 
As discussed in \S\ref{sec:algorithms}, our proposed framework \name leveraging Count Sketch satisfies the stronger local privacy model. 
We next analyze the differential privacy properties of Count Sketch. 

\begin{algorithm}[t]
        \begin{algorithmic}[1]
    	    \STATE {\bf Input:}  $g \in \mathbb{R}^n$, $t$, $k$,  $\bm{S}_{t \times k}$, $h_i~(1 \leq i \leq t)$, $\text{sign}_i~(1 \leq i \leq t)$
    	    \vspace{2mm}
    	    \STATE {\bf \underline{Compress vector $g \in \mathbb{R}^n$ into $\bm{S}(g)$:}}
    	    \STATE {Initialize $\bm{S}_{t\times k}$ to all zeros's}
    	    \FOR {$g_i \in {g}$}
	            \FOR  {$j=1, \cdots, t$}
	                \STATE ${\bm S}[j][h_j(i)] \leftarrow {\bm S}[j][h_j(i)] + \text{sign}_j(i) \cdot g_i$
	            \ENDFOR
	        \ENDFOR
	       \RETURN ${\bm{S}_{t\times k}}$
	       \vspace{2mm}
	       \STATE {\bf \underline{Query $\tilde{g} \in \mathbb{R}^n$ from $\bm{S}(g)$:}}
	       \FOR {$i = 1, \cdots, n$}
	            \STATE $\tilde{g_i}=\text{Median}\{\text{sign}_j(i) \cdot {\bm S}[j][h_j(i)]: 1\leq j \leq t\}$
	        \ENDFOR
	        \RETURN $\tilde{g}$
	  \end{algorithmic}
	  \caption{Count Sketch to compress $g \in \mathbb{R}^n$.}\label{alg:countsketch}
\end{algorithm}

\subsection{Differential Privacy Analysis}\label{sec:analysis:dp}
To analyze the privacy guarantees of sketches, we need to provide some measurement of the output distribution that reflects the amount of noise that results from hash collisions in the sketch. 
To achieve this, we quantify the input data distribution with the following assumptions: (a) each element in the input vector is bounded with a large probability, and (b) the inputs are drawn from Gaussian distributions. Under these assumptions, we provide the first differential privacy guarantees of  Count Sketch without any additional noise.

\begin{assumption}[Input vector distribution]\label{assum:basic} 
Following the work of~\citet{gong2014gradient,fergus2006removing,levin2007image,parmas2018total}, we assume that for any input vector $D$ with length $|D|=n$,
 each element $d_i \in D$ is drawn I.I.D. from a Gaussian distribution: $d_i \sim$ $\mathcal{N}(0,\sigma^2)$, and bounded by a constant with a large probability: $|d_i| \leq \alpha, 1 \leq i \leq n$, for some constant $\alpha$.
\end{assumption}

As discussed in \S\ref{sec:algorithms}, we apply Count Sketch to compress gradients (or model updates) in distributed learning. Therefore, the input vectors are gradients (or model updates) generated from the workers' local data.  
We note that similar assumptions on normality have been previously used to characterize gradients for other applications~\citep[e.g.,][]{gong2014gradient,fergus2006removing,levin2007image,parmas2018total}.
We also empirically verify that the gradients or model updates closely match Gaussian distributions, and plot the distribution of real updates in Figure~\ref{fig:mnist_grad} and~\ref{fig:shakespeare_grad} in Appendix~\ref{app:eval:grad}. As we will discuss, our local privacy bound is a function of the local gradients/model updates bound $\alpha$ and its variance $\sigma^2$. When leveraging sketches for privacy, we dynamically estimate $\alpha$ and $\sigma^2$ locally for each participating worker during training without prior knowledge, and use the estimated values to compute the privacy parameter $\varepsilon$.

Next, we state a lemma about Count Sketch's error bounds before proving the differential privacy in our main theorem.
\begin{lemma}[Estimation error of Count Sketch~\citep{CountSketch}]\label{lemma:cs_error}
For a sketching mechanism $\mathcal{M}$ using Count Sketch with $t$ arrays of $k$ bins, for any input element $d_i \in D$ and query $\mathcal{Q}$, with probability $p \geq 1-\delta$,
\begin{align*}
    |\mathcal{Q(M}(d_i))-d_i| \leq \mu \|D\|_2,
\end{align*}
where $k=O\left(\frac{e}{\mu^2}\right)$, and $t=O\left(\ln(\frac{1}{\delta})\right)$.
\end{lemma}
Lemma~\ref{lemma:cs_error} states the error guarantee of Count Sketch, i.e., that with high probability, the recovered/estimated values are very close to the original value, with a small error up to a fraction $\mu$ of $L_2$ norm of the input vector. With this lemma, we can translate the bound of the input values to the bound of every value in the output sketching table, which will be used in our proof.
We provide the differential privacy guarantees of a basic Count Sketch $\M$ in our main theorem below.

\begin{theorem}[$\varepsilon$-differential privacy of Count Sketch]\label{theorem:main}
For a sketching algorithm $\mathcal{M}$ using Count Sketch ${\bm S_{t \times k}}$ with $t$ arrays of $k$ bins, for any input vector $D$ with length $n$ satisfying Assumption~\ref{assum:basic}, $\mathcal{M}$ achieves  $t\cdot \ln\left(1+\frac{\beta \alpha^2 k(k-1)}{\sigma^2(n-2)}\left(1+\ln(n-k)\right)\right)$-differential privacy with high probability, where $\beta$ is a positive constant satisfying $\frac{\alpha^2k(k-1)}{\sigma^2(n-2)}\left(1+\ln(n-k)\right) \leq \frac{1}{2}-\frac{1}{\beta}$.
\end{theorem}

We defer readers to Appendix~\ref{app:privacy} for a detailed proof, and provide a high-level sketch here: Since the $t$ counter arrays are independent, we first  derive the output probability density functions of a single counter array.  Based on the sketching algorithm (Algorithm~\ref{alg:countsketch}) as well as the input distributions (Assumption~\ref{assum:basic}), we can show that the output distribution is a mixture of Gaussians. Since we assume uniformly bounded inputs, we can then prove that the difference of two output distributions before and after one input element changes is also bounded. The overall privacy guarantee is then multiplied by $t$ given the independence between  arrays.

To the best of our knowledge, we provide the first differential privacy proof for vanilla Count Sketch without additional randomization. From Theorem~\ref{theorem:main}, we see that the privacy parameter $\varepsilon$ will become smaller when the size of the input vector, $n$, gets smaller, and we can achieve nearly zero-differential privacy when $n$ is large enough. Theorem~\ref{theorem:main} also reveals  natural trade-offs between accuracy, communication, and privacy. For instance, if we use a smaller sketching table to compress the input vector (smaller $t$ or $k$), we are compressing more aggressively (losing more information) and  $\mathcal{M}$ will be more private with a smaller $\varepsilon$ parameter---potentially at the cost of overall accuracy.

\section{Proposed Framework: \name}\label{sec:algorithms}
In this section, we introduce \name, a general framework for distributed learning. We begin by describing \name's end-to-end approach for reducing communication and enforcing differential privacy using sketches (\S\ref{sec:algorithms:meta}). We then show how \name can be applied to both classical distributed machine learning, with distributed SGD (\S\ref{sec:algorithms:sgd}), and federated learning, with \fedavg (\S\ref{sec:algorithms:flearn}).

\subsection{\name: A Framework for Distributed Learning}\label{sec:algorithms:meta}
\name is an efficient and private framework for distributed learning. On top of the existing distributed  methods, \name uses sketches to both preserve the privacy and reduce the size of transmitted messages with small overhead. 
In particular, many distributed optimization methods perform work in parallel (e.g., computing gradients), and then aggregate local information from distributed entities by calculating a mean (e.g., computing the average mini-batch of gradients). These tasks can be completed using a sketch through the \emph{Aggregation} and \emph{Query} steps.   
As depicted in Algorithm~\ref{alg:diffsketch}, \name has two additional simple operations: \emph{Compression} and \emph{Validation}. 

The complete operations in \name are as follows:
   \emph{Compression:} At each round,  
   entity $k$ uses sketches to compress the local information $v$ into a privatized $\bm S(v_k)$.
   \emph{Validation:} Once the compression is done, each local worker needs to verify if sketching alone satisfies $\varepsilon$-differential privacy (based on Theorem~\ref{theorem:main}). If not, we may need to ensure a consistent privacy guarantee with some additional noise. As sketches have provable privacy, zero or a small amount of noise is expected (we verify this empirically in \S\ref{sec:eval}).
   \emph{Aggregation:} The central server aggregates the local information and sends the aggregated (i.e., averaged) local updates back.
   \emph{Query:} Each local entity approximately recovers the updated global information from the (merged) sketch.

\begin{algorithm}[t]
        \begin{algorithmic}[1]
    	    \STATE {\bf Input:}  $T$, $v_1, \cdots, v_m$, $\varepsilon$
    	    \FOR  {$t=0, \cdots, T-1$}
		    \STATE {\bf Compression:} Each entity $k$ compresses $v_k$ to obtain $\bm{S}(v_k)$ based on Algorithm~\ref{alg:countsketch}
		    \STATE {\bf Validation:} Each entity validates if it satisfies  $\varepsilon$-differential privacy for a given $\varepsilon$. If not, add appropriate Laplacian or Gaussian noise to $\bm{S}(v_k)$ 
		    \STATE {\bf Aggregation:} The server aggregates local information to obtain $\bm{S}(v)=\frac{1}{m}(\bm{S}_1 + \cdots + \bm{S}_m)$ 
		    \STATE {\bf Query:} Each entity queries $\bm{S}(v)$ for the mean $\tilde{v}$ based on Algorithm~\ref{alg:countsketch}
		    \ENDFOR
	  \end{algorithmic}
	  \caption{Proposed framework: \name.}\label{alg:diffsketch}
\end{algorithm}

As mentioned in the related work (\S\ref{sec:related_work:privacy}) and preliminaries (\S\ref{sec:analysis:preli}), local differential privacy does not assume a trusted central server, and the server cannot see individual inputs. Since \name allows each worker to sketch the vectors locally before sending the compressed sketching table to the server, we are privatizing data at a local level, thus protecting against third-parties including the central server. 

\paraf{Generality.} We note that \name is a general framework for a variety of distributed learning scenarios where user data needs to be protected. In this paper, we describe the applications to distributed SGD (\S\ref{sec:algorithms:sgd}) and federated learning (\S\ref{sec:algorithms:flearn}). It can also be potentially applied to broader distributed data analysis tasks with privacy concerns; the main requirement is simply that the workload is aggregating information across the network via calculation of a mean of distributed vectors.  
\name's modular design has two major advantages: First, the provable accuracy guarantees and differential privacy for the compressed data help any analysis algorithms to retain high accuracy. Second, the simple hashing-based computations in applying the sketches to Compression and Query are light-weight operations. One can leverage efficient CPU parallelism or hardware acceleration (e.g., SSE and AVX) to achieve optimized computation performance.

\subsection{\name for Distributed SGD} \label{sec:algorithms:sgd}
We first explore \name with distributed (mini-batch) SGD as a subroutine. Throughout the paper, we assume the widely-used objective of minimizing the finite sum of the empirical loss:
\begin{align}
    \min_w F(w) = \sum_{k=1}^m p_k F_k(w) \, , \label{eq:obj}
\end{align}
where $F_k$ is the local empirical loss on worker $k$, $m$ is the total number of workers, and $p_k$ is the weight set for worker $k$ (e.g., $\frac{1}{m}$). We consider possibly differing local data distributions across the network and possibly non-convex $F_k$'s. In this work, we assume a classical synchronous and centralized training setup with $m$ workers connected to one central server. 

A typical workflow of using distributed SGD to solve \eqref{eq:obj} is that at each updating round, each local worker computes gradients using (a mini-batch of) local data and sends the gradients to the server, where they get merged and sent back to the workers. We can directly apply Algorithm~\ref{alg:diffsketch} to compress the gradients ${g}$'s sent from workers to the server. The server can not see the raw gradients as they have been compressed and masked into small sketching tables. 
See Algorithm \ref{alg:sgd_sketch} for more details. As sketches are guaranteed to preserve the original skewed gradient values with high probabilities, the local workers can recover the merged gradients with very high accuracies. In our experiments (\S\ref{sec:eval}), we demonstrate that the accuracy reduction is very minor while the compression ratio is high (up to $50\times$).

\paraf{Convergence.} We also provide convergence guarantees for \name with distributed SGD for convex functions. 
Our convergence results rely on the bounded estimation error of Count Sketch (Lemma~\ref{lemma:cs_error}), and unbiasedness of Count Sketch~\citep{CountSketch}: For a Count Sketch $\mathcal{M}$, for any input element $d_i \in D$ and query $\mathcal{Q}$, $\mathbb{E}\left[\mathcal{Q}(\mathcal{M}(d_i)\right] = d_i.$
Together with Lemma~\ref{lemma:cs_error}, this 
guarantees that the estimated gradients are both unbiased and uniformly bounded with high probability.

\begin{theorem}[Convergence of \name in distributed SGD]\label{theorem:fl_convergence}
Assume that $E\left[\|g_k\|^2\right] \leq G^2$ for any input stochastic gradient $g_k$ on device $k$ with dimension $n$, $E\left[\|x_i - x^*\|^2\right] \leq D^2$ at any iteration $i$, and the local model at device $k$ $f_k(x)$ is convex. Choose the step-size at round $i$  $\eta_i=\frac{c}{\sqrt{i}}$ where $c$ is a pre-defined positive number. Using a Count Sketch with $t$ hash functions and $k$ bins, with probability ($1-\delta$), we have:
\begin{equation*}
    F(\bar{x}_i)-F(x^*) \leq \frac{\frac{D^2}{2c}+c\sqrt{\frac{i+1}{i}}\left(n\mu^2+1\right)G^2}{\sqrt{i}},
\end{equation*}
where $x^*$ is the optimal solution to \eqref{eq:obj}, $k=O\left(\frac{e}{\mu^2}\right)$, $t=O\left(\ln(\frac{1}{\delta})\right)$, and $\bar{x}_i$ = $\frac{1}{i} \sum_{j=1}^i x_j$.
\end{theorem}

For simplicity, we abuse notations and use $g_k \in \mathbb{R}^n$ to denote the gradient vector on device $k$. We provide a detailed proof in Appendix~\ref{app:convergence}. Theorem~\ref{theorem:fl_convergence} indicates that \name has the same convergence rate as standard distributed SGD~\citep{ghadimi2013stochastic} under convex settings. It also indicates that as the number of hash functions $t$ and the number of bins $k$ increase, we compress less and tend to get a tighter convergence bound (due to a smaller recovery error $\mu$) with a higher probability (lower $\delta$).

\subsection{\name for Federated Learning}\label{sec:algorithms:flearn}

Federated learning aims to fit a model to data generated by, and residing on, networks of hundreds to  millions of remote devices~\citep{mcmahan2016FedAvg}. Applying \name must carefully consider unique challenges associated with this setting---the large scale of the networks,   expensive communication, strict privacy requirements, and a high degree of heterogeneity across devices~\citep{li2019federated}. In practice, it is common that only a small fraction of devices are active at each round~\citep{bonawitz2019towards}. 
Optimization methods using local updating and tolerating low participation of devices have become the de facto solvers for federated settings; of these, \fedavg is most widely-used~\citep{mcmahan2016FedAvg}. 

\name can use \fedavg as a subroutine and similarly account for important characteristics in federated learning. At each communication round, it randomly samples a subset of devices, lets each participating device perform $E$ epochs of local updates, applies Count Sketch to compress the updates, and  averages the updates centrally.  Details are summarized in Algorithm~\ref{alg:flearn_sketch}. 
Empirically, we demonstrate that \name can compress communication by up to $20\times$ and provide strong local privacy guarantees ($\varepsilon=1$) on real federated datasets (\S\ref{sec:eval}).

\begin{algorithm}[t]
        \begin{algorithmic}[1]
    	    \STATE {\bf Input:}  $T$, $\eta$, $w^0$, $\varepsilon$
	        \FOR  {$t=0, \cdots, T-1$}
		        \IF {$t>0$}
		        \STATE Server sends the sketched global gradient $\bm{S}(g^t)$ to all workers
		        \STATE Each worker queries $\bm{S}(g^t)$ for $\tilde{g}^t$
		        \STATE Each worker updates: $w^t = w^{t-1} - \eta \tilde{g}^t$
		        \ENDIF
		        \STATE Each worker $k$ runs (mini-batch) SGD on $w^t$ to obtain local gradients ${g}_k^{t+1}$
		        \STATE Each worker sketches the gradients locally to obtain $\bm{S}({g}_k^{t+1})$
		        \STATE Each worker adds additional Laplacian noise to $\bm{S}({g}_k^{t+1})$ if not satisfying $\varepsilon$-differential privacy
		        \STATE Each worker sends $\bm{S}({g}_k^{t+1})$ to the server
		        \STATE Server aggregates the model updates:  { $\bm{S}(g^{t+1}) = \frac{1}{m}\sum_{k=1}^m \bm{S}(g_k^{t+1})$}
	    \ENDFOR
	  \end{algorithmic}
	  \caption{\name with distributed SGD.}\label{alg:sgd_sketch}
\end{algorithm}

\begin{algorithm}[t]
        \begin{algorithmic}[1]
    	    \STATE {\bf Input:}  $K$, $T$, $\eta$, $E$, $w^0$,  $p_k$, $\varepsilon$
	        \FOR  {$t=0, \cdots, T-1$}
		        \STATE Server samples a subset $S_t$ of $K$ devices  (each device is chosen with probability $p_k$)
		        \IF {$t>0$}
		        \STATE Server sends the sketched global model $\bm{S}(\Delta w^t)$ to all chosen devices
		        \STATE Each device $k$ queries $\bm{S}(\Delta w^t)$ for $\Delta \tilde{w}^{t}$
		        \STATE Each device $k$ updates: $w^t = w^{t-1} + \Delta w^t$
		        \ENDIF
		       \STATE Each device $k$ updates $w^t$ for $E$ epochs of SGD on $F_k$ with step-size $\eta$ to obtain $\Delta w_k^{t+1}$
		        \STATE Each device $k$ sketches the updates locally to obtain $\bm{S}(\Delta w_k^{t+1})$
		         \STATE Each device $k$ adds additional Laplacian noise to $\bm{S}(\Delta w_k^{t+1})$ if not satisfying $\varepsilon$-differential privacy
		        \STATE Each device $k$ sends $\bm{S}(\Delta w_k^{t+1})$ to the server
		        \STATE Server aggregates the model updates: { $\bm{S}(\Delta w^{t+1}) = \frac{1}{K}\sum_{k \in S_t} \bm{S}(\Delta w_k^{t+1})$}
	    \ENDFOR
	  \end{algorithmic}
	  \caption{\name in federated learning.}\label{alg:flearn_sketch}
\end{algorithm}

We note that in federated settings with non-identically distributed data, \fedavg is a heuristic and may not converge despite its overall robust practical performance~\citep{mcmahan2016FedAvg,li2018federated}. Therefore, we also do not provide convergence guarantees for \name in this setting, though we explore the method empirically in \S\ref{sec:eval}.

\section{Evaluation}\label{sec:eval}

We now present empirical results for the \name framework. 
In \S\ref{sec:eval:compare}, we compare \name with other baselines that can reduce communication and preserve privacy simultaneously, and demonstrate the superior performance of \name. In \S\ref{sec:eval:tradeoff}, we investigate the trade-off between privacy, communication, and accuracy in \name. 
All code, data, and experiments are publicly available at \href{https://github.com/litian96/DiffSketch}{\texttt{github.com/litian96/DiffSketch}}.

\subsection{Simulation Setups}\label{sec:eval:setup}
{\bf Datasets.} For distributed SGD, we randomly subsample from the MNIST dataset for image classification 
and form 10 partitions with identical distributions across local workers. For federated learning, we use the Shakespeare dataset, which is a popular dataset curated from federated learning benchmarks~\citep{caldas2018leaf}. It is naturally partitioned in a heterogeneous way where each device is associated with a speaking role in the plays. To investigate a setup that is favorable to \cpsgd~\citep{agarwal2018cpsgd}, we partition MNIST into 6,000 subsets with one subset corresponding to a device,  
as we observe that the privacy bound in \cpsgd becomes tighter (better) if sampling from a larger number of devices. Data statistics are summarized in Table~\ref{table: data}, Appendix~\ref{app:eval}.

{\bf Implementation.} We implement Algorithm~\ref{alg:sgd_sketch} and \ref{alg:flearn_sketch} in Tensorflow~\citep{abadi2016tensorflow}, and simulate a one server and $m$ workers setup (see Table~\ref{table: data} in Appendix~\ref{app:eval}). To further improve the performance of sketches, we consider several simple implementation optimizations as follows. In distributed SGD, each local worker compares the queried results with the local gradients, and sets half of the gradients with inaccurate estimations (larger gaps) to zero. 
For both distributed SGD and federated learning, we generate some random noise drawn from the same Gaussian distribution as the original gradients (or model updates), and append the noise vector  after the input. By increasing the size of the input and thus compressing more values, we are boosting the privacy of sketches. All $\varepsilon$ values reported in the following experiments are local privacy guarantees---each local worker achieves $\varepsilon$-differential privacy at each round. 
See full details in Appendix~\ref{app:eval:implementaion}.

\subsection{Comparison with Baselines}\label{sec:eval:compare}
In this section, we compare with several baseline methods that address both communication and privacy, including  \cpsgd~\citep{agarwal2018cpsgd}, a state-of-the-art method that aims optimize both aspects (though it treats each separately). As secure multi-party computation (SMC) is generally inefficient in distributed learning (as discussed in \S\ref{sec:related_work}), we do not compare with methods based on SMC. 

\begin{figure}[h]
    \centering
    \includegraphics[width=0.55\textwidth]{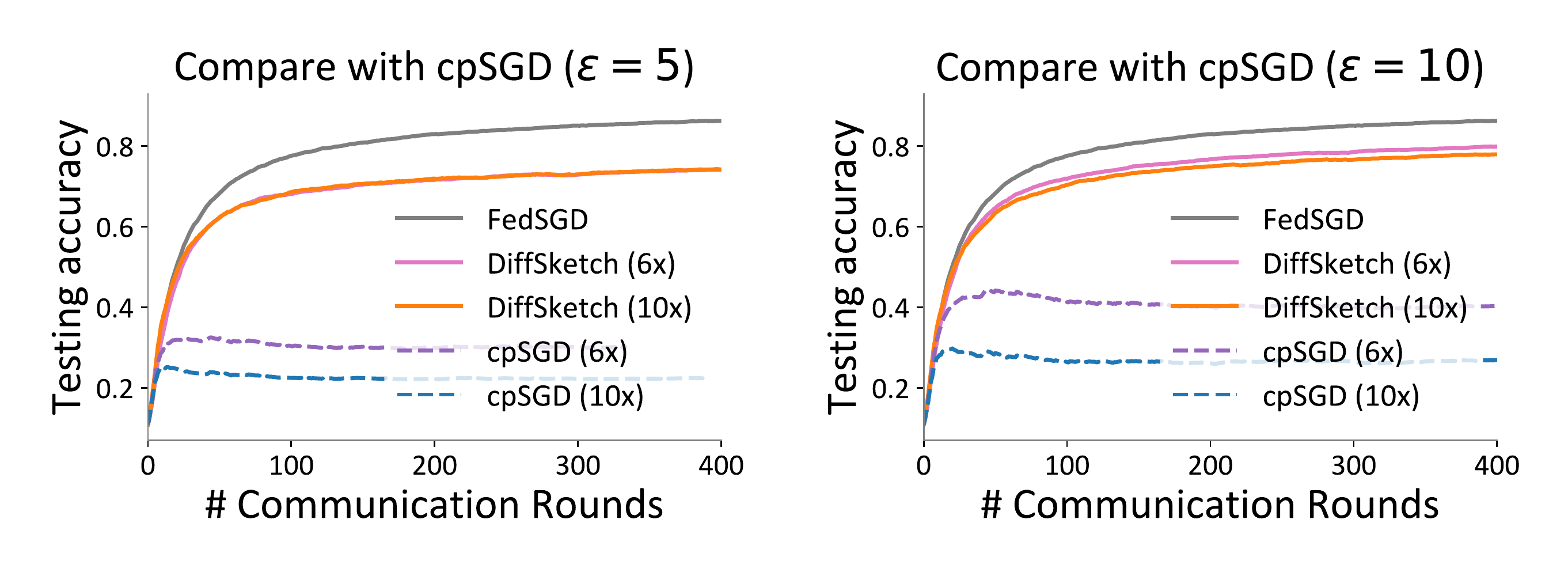}
    \caption{\name compared with \cpsgd. Under the same privacy guarantees ($\varepsilon=5$ or 10) and the same communication compression ratio, \name achieves significantly higher accuracy. If we continue to decrease $\varepsilon$ to improve the privacy, \cpsgd will perform even worse.}
    \label{fig:compare_cpsgd}
\end{figure}

{\bf cpSGD.} 
\cpsgd~\citep{agarwal2018cpsgd} first quantizes the gradients to reduce communication, then adds Binomial noise to the quantized gradients to offer differential privacy.
We compare \name with \cpsgd on MNIST*. Results are shown in Figure~\ref{fig:compare_cpsgd}. We see that across different $\varepsilon$ values and compression ratios, by considering privacy and communication jointly, \name achieves significant improvements of up to 50\% in \textit{absolute} test accuracy over \cpsgd. 

\begin{figure}[h]
	\centering
	\includegraphics[width=1\textwidth]{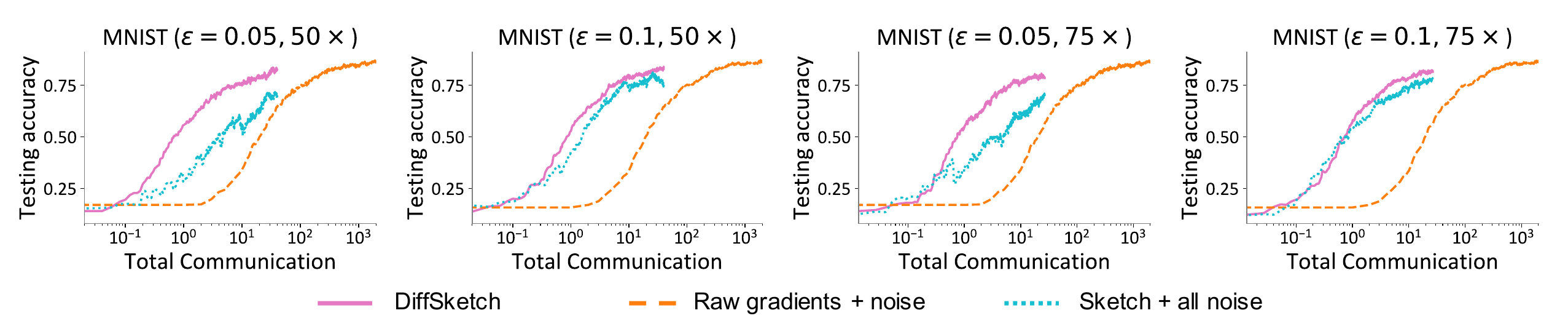}
	\caption{
	\name compared with other baselines in the distributed SGD setting. We show the test accuracy versus total communication in log scale. The total amount of communication is normalized via dividing the communication rounds by the compression ratio.  Given the same amount of privacy, (1) \name converges with orders-of-magnitude less communication than directly adding noise to the gradients (orange line), and (2) \name is more accurate than treating sketches as plain-text and directly adding noise to sketches at all communication rounds (blue line). This indicates that \name is inherently private such that less (or zero) additional noise is sufficient to provide the same privacy guarantees. The sketch sizes of 50$\times$ and 75$\times$ compression ratios are $7 \times 22$ and $7 \times 15$, respectively.}
	\label{fig:mnist}
\end{figure}

\begin{figure}[h]
	\centering
	\includegraphics[width=1\textwidth]{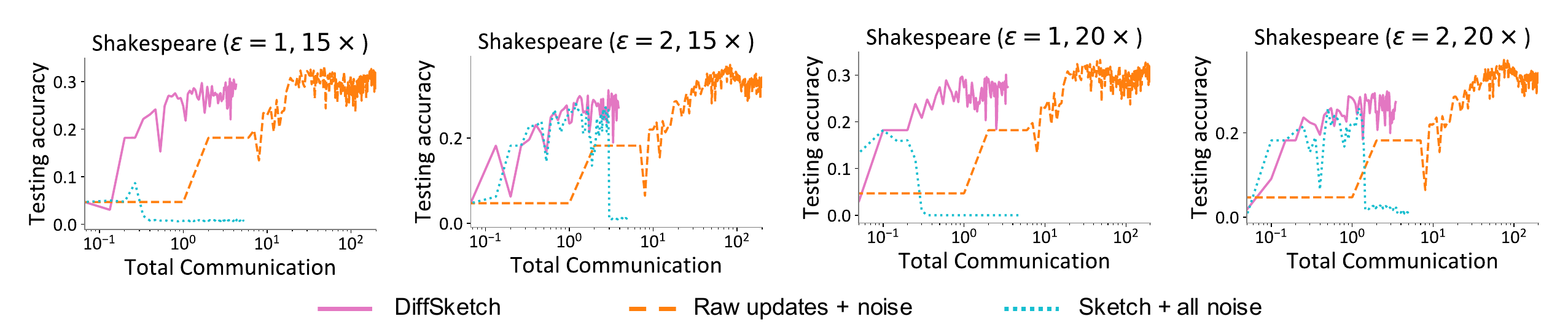}
	\caption{\name compared with other baselines in federated learning. Again, we note here that the total amount of communication is in log scale. Similar to Figure~\ref{fig:mnist}, it indicates that \name can significantly improve test accuracy while offering the same privacy and communication benefits. We observe that smaller $\varepsilon$ values would cause baseline methods to diverge, and so we investigate slightly larger $\varepsilon$'s. The sketch sizes of 15$\times$ and 20$\times$ compression ratios are $10 \times 190$ and $10 \times 245$. }
	\label{fig:shakespeare}
\end{figure}

{\bf Other baselines.} We next compare with two natural alternatives that optimize for privacy in practice: (1) directly adding Laplacian noise to the raw gradients transmitted between the workers and the central server to offer (local) differential privacy, and (2) treating sketches as plain-text, and adding Laplacian noise after sketching at all updating rounds to offer both privacy and compression~\citep{MelisDC16}. The results are shown in Figure~\ref{fig:mnist} (for distributed SGD) and Figure~\ref{fig:shakespeare} (for federated learning). In both settings, we see that \name converges faster than adding noise to raw gradients as the amount of communication is significantly reduced by sketching. Also, \name is more accurate than adding additional noise to sketches at all iterations, especially when $\varepsilon$ is smaller. This is because in order to guarantee stronger privacy with smaller $\varepsilon$'s, the baseline approach (in blue) needs to add more noise on top of sketches, thus hurting model performance, while \name only adds a small amount of noise when necessary (Figure~\ref{fig:eps} in the appendix).

\subsection{Trade-offs in \name}\label{sec:eval:tradeoff}

Finally, we explore trade-offs in \name between communication, privacy, and accuracy.

{\bf Accuracy vs.\   compression.} We plot the testing accuracy with the total amount of communication under different compression ratios for both distributed SGD and federated learning in Figure~\ref{fig:accu_compression}. As the compression ratio increases, we obtain faster convergence with less communication, but have potentially lower final accuracy. We note that many existing works in model compression explore techniques for improving accuracy under high compression ratios~\citep[e.g.,][]{konevcny2016federated,lin2017deep}; although outside the focus of this work, these techniques could be combined with our work to further boost the compression ratios.

{\bf Privacy vs.\   communication.} We show how compression relates to differential privacy in Figure~\ref{fig:privacy_communication}. The $\varepsilon$ values are averaged across all communication rounds. In our experiments, sketches themselves are sufficient to provide privacy benefits without additional noise in most communication rounds. We can see that a higher compression ratio leads to more privacy (smaller $\varepsilon$). 

\begin{figure}[h]
    \centering
    \includegraphics[width=0.6\textwidth]{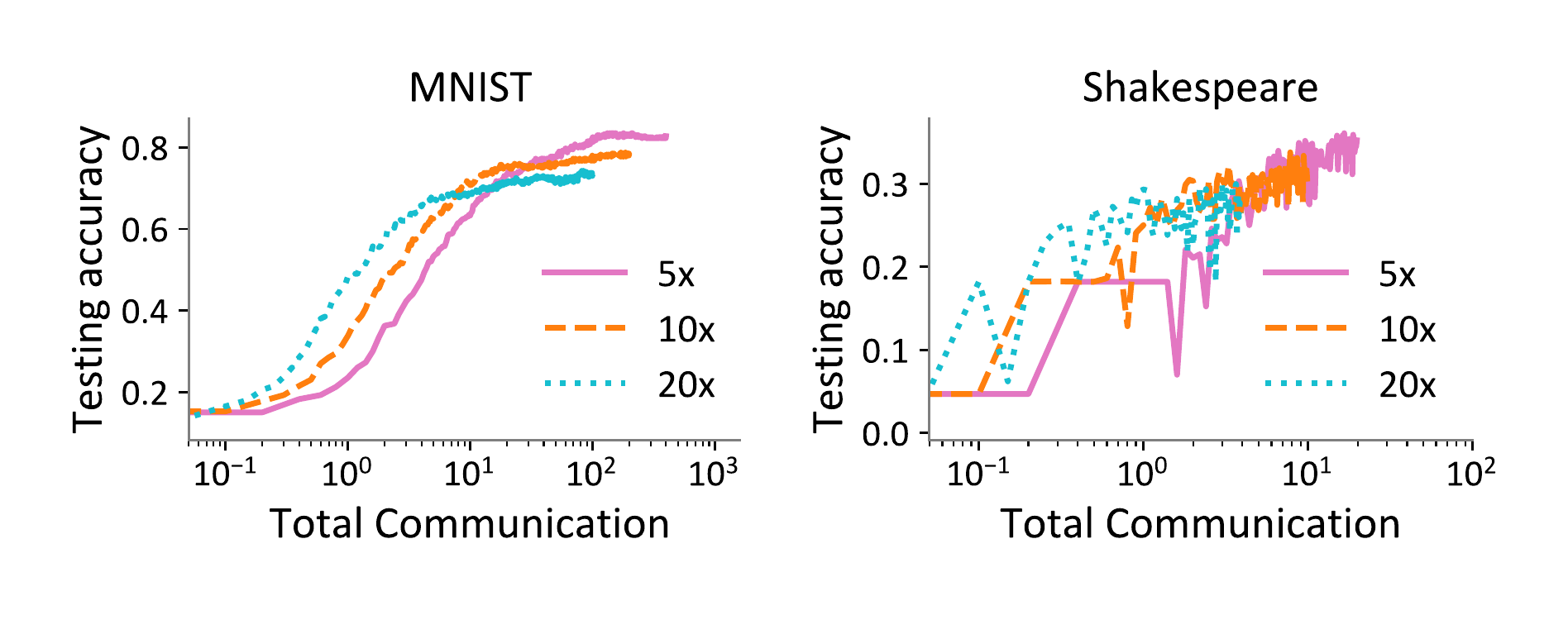}
    \caption{ As the compression ratio becomes higher, \name converges faster, but with potentially lower accuracies.}
    \vspace{-4mm}
    \label{fig:accu_compression}
\end{figure}
\vspace{-2mm}
\begin{figure}[h]
    \centering
    \includegraphics[width=0.55\textwidth]{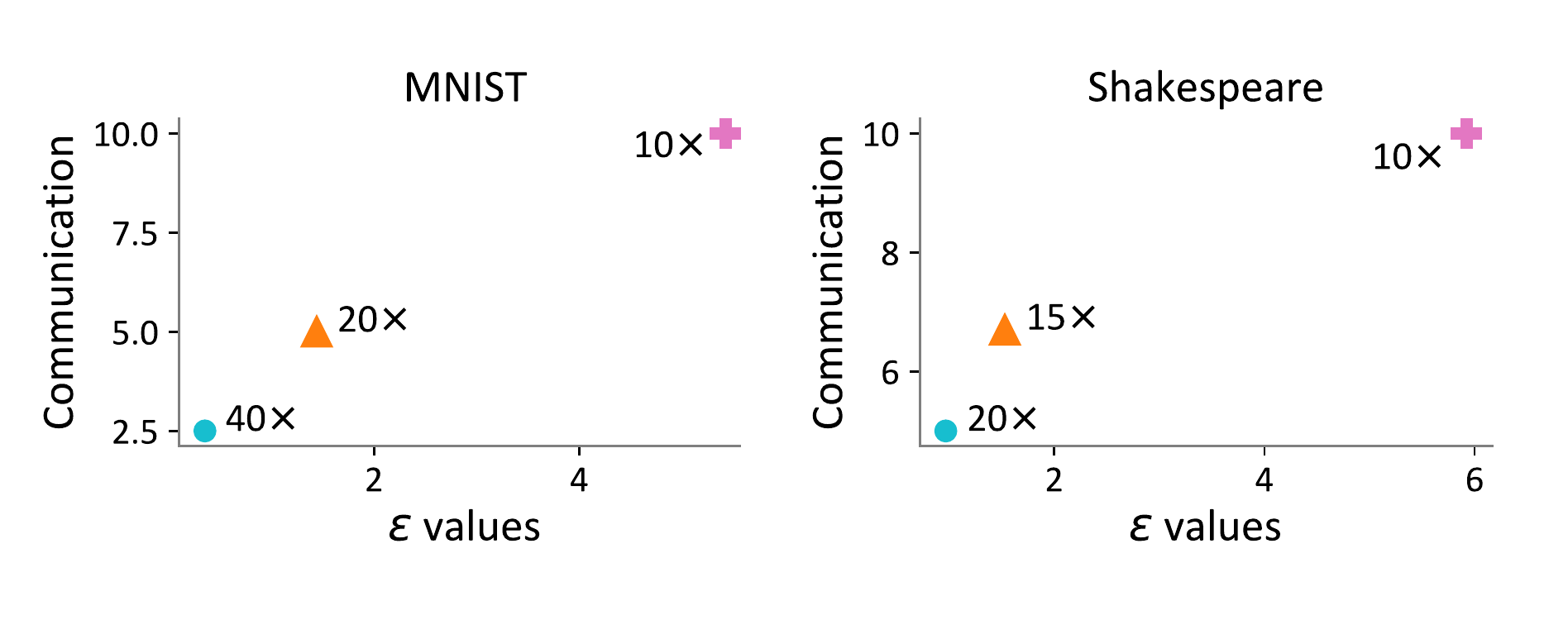}
    \vspace{-4mm}
    \caption{As we compress more, we can reduce communication and obtain stronger privacy guarantees with smaller $\varepsilon$ values. }
    \label{fig:privacy_communication}
    \vspace{-2mm}
\end{figure}

\section{Conclusion}\label{sec:conclusion}\vspace{-2mm}
In this work, we examined connections between communication reduction and privacy preservation in distributed machine learning. 
We first proved that canonical sketches (Count Sketch) have inherent differential privacy benefits without additional mechanisms. Based on our theoretical understandings, we designed \name, a framework for efficient and private distributed learning. We  applied \name to  classical distributed SGD and federated learning, and demonstrated empirically that \name achieves 5\%-50\% absolute accuracy improvements compared with baselines while offering the same communication and privacy benefits. 
While we explored such connections via sketches, a natural direction of future work is to investigate whether other tools from the privacy or distributed learning communities may similarly provide benefits for both privacy and communication simultaneously.

\section*{Acknowledgement}
We thank  Jalaj Upadhyay and Weizhao Tang for their helpful discussions.
This work was supported in part by the National Science Foundation grant IIS1838017, CNS-1700521, CNS-1565343, a Google
Faculty Award, a Carnegie Bosch Institute Research Award, Intel Labs University Research Office, and the CONIX Research Center. Any opinions,
findings, and conclusions or recommendations expressed in this material are those of the author(s) and do
not necessarily reflect the National Science Foundation or any other funding agency.

\bibliographystyle{abbrvnat}
\bibliography{ref,alan}

\newpage
\onecolumn
\appendix
\section{Proof for Theorem~\ref{theorem:main}}\label{app:privacy}

{\bf Theorem}: ($\varepsilon$-differential privacy of Count Sketch)
For a sketching algorithm $\mathcal{M}$ using Count Sketch ${\bm S_{t \times k}}$ with $t$ arrays of $k$ bins, for any input vector $D$ with length $n$ satisfying Assumption~\ref{assum:basic}, $\mathcal{M}$ achieves  $t \ln\left(1+\frac{\beta \alpha^2 k(k-1)}{\sigma^2(n-2)}\left(1+\ln(n-k)\right)\right)$-differential privacy with high probability, where $\beta$ is a positive constant satisfying $\frac{\alpha^2k(k-1)}{\sigma^2(n-2)}\left(1+\ln(n-k)\right) \leq \frac{1}{2}-\frac{1}{\beta}$.

{\bf Notations.} We first list the notations that will be used throughout the proof.
\begin{itemize}
    \item Input vector: $D=\{d_1, \cdots, d_n\},~|D|=n$
    \item The constructed data structure: $\bm S$ 
    \item Number of bins in one counter array (i.e., one row of $\bm S$): $k$
    \item Number of hash functions (number of counter arrays): $t$
    \item Random variables of the bin values in one counter array: $B_1, \cdots, B_k$
    \item Bin values in one counter array: $b_1, \cdots, b_k$
    \item Random variables of numbers of elements in $D$ mapped to $k$ bins: $A_1, \cdots, A_k$
    \item Numbers of elements in $D$ mapped to $k$ bins: $a_1, \cdots, a_k$
\end{itemize}

We state one lemma that will  be used in the proof for Theorem~\ref{theorem:main}.
\begin{lemma}\label{lemma:tech_1}
Given $\sum_{i=1}^k a_i=n-1$, $a_i \in \mathbb{Z}^+$, $a_1!\cdot a_2!\cdots a_k! \cdot \sqrt{a_1\cdot a_2 \cdots a_k}$ is minimized when $\max\{|a_i-a_j|, 1\leq i,j \leq k\} \leq 1$; also $\frac{1}{a_1}+\cdots+\frac{1}{a_k}$ is minimized when $\max\{|a_i-a_j|, 1\leq i,j \leq k\} \leq 1$.\label{eq}
\end{lemma}

\begin{proof}
Let $f(a_1, a_2, \cdots, a_k) = a_1!\cdot a_2!\cdots a_k! \cdot \sqrt{a_1\cdot a_2 \cdots a_k}$. If there exists $a_j \neq a_i$ and $a_i + 2 \leq a_j$, then $f(a_1, \cdots, a_i + 1, \cdots, a_j-1, \cdots, a_k) < f(a_1, a_2, \cdots, a_k)$. This is because
\begin{align*}
    \frac{f(a_1, a_2, \cdots, a_k)}{f(a_1, \cdots, a_i + 1, \cdots, a_j-1, \cdots, a_k)} &= \sqrt{\frac{a_i a_j^3}{(a_i+1)^3(a_j-1)}} \\ &\overset{a_j=a_i+k}{=} \sqrt{\frac{a_i(a_i+k)^3}{(a_i+1)^3(a_i+k-1)}} \\
    &= \sqrt{\frac{a_i^4+3ka_i^3+3k^2a_i^2+k^3a_i}{a_i^4+(2+k)a_i^3+3ka_i^2+(3k-2)a_i+k-1}} >1,
\end{align*}
where the last inequality is due to 
\begin{align*}
    (2k-2)a_i^3+ (3k^2-3k)a_i^2+(k^3-3k+2)a_i \geq 1.
\end{align*}
We can adjust $a_i$ to $a_i+1$, $a_j$ to $a_j-1$ to get a lower value of $f$. After limited times of such adjustments, $f(a_1, \cdots, a_k)$ will be minimized when the integers $a_1, a_2, \cdots, a_k$ satisfy $\max\{|a_i-a_j|, 1\leq i,j \leq k\} \leq 1$.

Similarly, we can easily extend the above analysis to $\frac{1}{a_1}+\cdots+\frac{1}{a_k}$ and prove that it gets minimized if $\max\{|a_i-a_j|, 1\leq i,j \leq k\} \leq 1$.
\end{proof}

We next prove our main results Theorem~\ref{theorem:main}.
\begin{proof}
We first prove that when $t=1$, the sketching algorithm using ${\bm{S}_{1\times k}}$ with one array and $k$ bins (i.e., $k$ counters in the array) achieves  $\ln\left(1+\frac{\beta \alpha^2 k(k-1)}{\sigma^2(n-2)}(1+\ln(n-k))\right)$-differential privacy with high probability,  where $\beta$ is a constant satisfying $\frac{\alpha^2k(k-1)}{\sigma^2(n-2)}\left(1+\ln(n-k)\right) \leq \frac{1}{2}-\frac{1}{\beta}$.

Suppose two input vectors $D$ and $D'$ differ in one element with $D = {D_0}+\{c\}$ and $D' = {D_0}+\{d\}$. Since the output is independent of the order of compressing each element in $D$,  we consider mapping $c$ (or $d$) at last. $c$ (or $d$) is mapped to each bin with equal probability, and they might be flipped by the sign hash functions with a probability $\frac{1}{2}$. Therefore, 
\begin{align*}
    P(b_1, b_2, \cdots, b_k | D) &= \frac{1}{2k}P(b_1-c, b_2, \cdots, b_k | D_0) + \frac{1}{2k}P(b_1+c, b_2, \cdots, b_k | D_0) + \cdots \\ &+ \frac{1}{2k}P(b_1, b_2, \cdots, b_k+c | D_0) + \frac{1}{2k}P(b_1, b_2, \cdots, b_k-c | D_0), \\
    P(b_1, b_2, \cdots, b_k | D') &= \frac{1}{2k}P(b_1-d, b_2, \cdots, b_k | D_0) + \frac{1}{2k}P(b_1+d, b_2, \cdots, b_k | D_0) + \cdots \\ &+ \frac{1}{2k}P(b_1, b_2, \cdots, b_k+d | D_0) + \frac{1}{2k}P(b_1, b_2, \cdots, b_k-d | D_0).
\end{align*}

Denote $P_{(a_1, \cdots, a_k)}^n$ as the probability that $a_i$ $(1 \leq i \leq k)$ items are mapped to bin $i$ and $\sum_{i=1}^k a_i=n$. Denote $G_{(a_1, \cdots, a_k)}^n$ as the probability density function of a $k$-dimensional Gaussian $\mathcal{N}\left(0, \begin{bmatrix} 
\sigma^2 a_1 &         &  \\
    & \ddots &  \\
    &        & \sigma^2 a_k 
\end{bmatrix}\right)$ and $\sum_{i=1}^k a_i=n$. Note that 
\begin{align*}
    P(b_1-c, b_2, \cdots, b_k | D_0) = \sum_{\{a_1, \cdots, a_k\}} P_{(a_1, \cdots, a_k)}^{n-1} G_{(a_1, \cdots, a_k)}^{n-1} (b_1-c, b_2, \cdots, b_k). 
\end{align*}
\vspace{-2mm}
We first consider the upper-bound of:
\begin{align*}
    \frac{P(b_1-c, b_2, \cdots, b_k | D_0)}{P(b_1-d, b_2, \cdots, b_k | D_0)} = \frac{\sum_{\{a_1, \cdots, a_k\}} P_{(a_1, \cdots, a_k)}^{n-1} G_{(a_1, \cdots, a_k)}^{n-1} (b_1-c, b_2, \cdots, b_k)}{\sum_{\{a_1, \cdots, a_k\}} P_{(a_1, \cdots, a_k)}^{n-1} G_{(a_1, \cdots, a_k)}^{n-1} (b_1-d, b_2, \cdots, b_k)}, 
\end{align*}
where the joint probability distribution of number of items when putting $n-1$ items into $k$ bins $\{b_1, \cdots, b_k\}$ is denoted as $P_{(a_1, \cdots, a_k)}^{n-1} = \frac{(n-1)!}{k^{n-1}\cdot a_1! \cdots a_k!}$. We assume that each bin has at least one item, so there are $n-2 \choose k-1$ ways in total to put $n-1$ items into $k$ bins. Note that probability that each bin has at least one item given $k$ bins and $n-1$ items is $\sum_{j=0}^k (-1)^j {k \choose j} \left(1-\frac{j}{k}\right)^{n-1}$, so the following differential privacy properties with one hash function hold with a high probability at most $\sum_{j=0}^k (-1)^j {k \choose j} \left(1-\frac{j}{k}\right)^{n-1}$.

Denote $\mathbf{x}$ as $\{b_1-c, b_2, \cdots, b_k\}$, since $e^{-\frac{1}{2} \mathbf{x}^\intercal \Sigma^\intercal \mathbf{x}} \leq 1$, we have:
\begin{align}
    & \frac{\sum_{\{a_1, \cdots, a_k\}} P_{(a_1, \cdots, a_k)}^{n-1} G_{(a_1, \cdots, a_k)}^{n-1} (b_1-c, b_2, \cdots, b_k)}{\sum_{\{a_1, \cdots, a_k\}} P_{(a_1, \cdots, a_k)}^{n-1} G_{(a_1, \cdots, a_k)}^{n-1} (b_1-d, b_2, \cdots, b_k)} \\
    \leq & \frac{\sum_{\{a_1, \cdots, a_k\}}\frac{1}{a_1! \cdot a_2! \cdots a_k!}\frac{1}{\sqrt{a_1\cdot a_2 \cdots \cdot a_k}}}{\sum_{\{a_1, \cdots, a_k\}}\frac{1}{a_1! \cdot a_2! \cdots a_k!}\frac{1}{\sqrt{a_1\cdot a_2 \cdots \cdot a_k}} e^{\left[-\frac{1}{2\sigma^2}\left(\frac{(b_1-d)^2}{a_1}+\frac{(b_2)^2}{a_2}+\cdots+\frac{(b_k)^2}{a_k}\right)\right]}}\nonumber \\ 
    \leq & \frac{\sum_{\{a_1, \cdots, a_k\}}\frac{1}{a_1! \cdot a_2! \cdots a_k!}\frac{1}{\sqrt{a_1\cdot a_2 \cdots \cdot a_k}}}{\sum_{\{a_1, \cdots, a_k\}}\frac{1}{a_1! \cdot a_2! \cdots a_k!}\frac{1}{\sqrt{a_1\cdot a_2 \cdots \cdot a_k}} e^{\left[-\frac{1}{2\sigma^2}\left(\frac{\gamma}{a_1}+\frac{\gamma}{a_2}+\cdots+\frac{\gamma}{a_k}\right)\right]}}\label{eq:5},
\end{align}
where $\gamma=\max\{(b_1-d)^2, b_2^2, \cdots, b_k^2\}$. From the bounded estimation errors of sketching (Lemma~\ref{lemma:cs_error}), we know that $b_i^2 \leq \alpha^2~(\forall i)$ with a large probability. So $ \gamma \leq 4\alpha^2$ holds with high probability.

Note that \eqref{eq:5} can be viewed as an multiplicative inverse of weighted sum on $e^{\left[-\frac{1}{2}\left(\frac{\gamma}{a_1}+\frac{\gamma}{a_2}+\cdots+\frac{\gamma}{a_k}\right)\right]}$ with weights $\frac{1}{a_1!\cdot a_2! \cdots a_k!} \frac{1}{\sqrt{a_1\cdot a_2\cdots a_k}}$. From Lemma \ref{lemma:tech_1}, we know that both $e^{\left[-\frac{1}{2}\left(\frac{\gamma}{a_1}+\frac{\gamma}{a_2}+\cdots+\frac{\gamma}{a_k}\right)\right]}$ and the weights $\frac{1}{a_1!\cdot a_2! \cdots a_k!} \frac{1}{\sqrt{a_1\cdot a_2\cdots a_k}}$ are maximized simultaneously when $a_1, a_2, \cdots, a_k$ satisfies $\max\{|a_i-a_j|, 1\leq i, j \leq k\} \leq 1$. If we set all weights to be equal, then the inverse of the weighted sum will become larger, w.h.p. Therefore,
\begin{align}
    \frac{\sum_{\{a_1, \cdots, a_k\}}\frac{1}{a_1! \cdot a_2! \cdots a_k!}\frac{1}{\sqrt{a_1\cdot a_2 \cdots \cdot a_k}}}{\sum_{\{a_1, \cdots, a_k\}}\frac{1}{a_1! \cdot a_2! \cdots a_k!}\frac{1}{\sqrt{a_1\cdot a_2 \cdots \cdot a_k}} e^{\left[-\frac{1}{2\sigma^2}\left(\frac{\gamma}{a_1}+\frac{\gamma}{a_2}+\cdots+\frac{\gamma}{a_k}\right)\right]}} &\leq \frac{{n-2 \choose k-1}}{\sum_{\{a_1, \cdots, a_k\}}e^{\left[-\frac{1}{2\sigma^2}\left(\frac{\gamma}{a_1}+\frac{\gamma}{a_2}+\cdots+\frac{\gamma}{a_k}\right)\right]}} \nonumber \\
    &\leq \frac{{n-2 \choose k-1}}{\sum_{\{a_1, \cdots, a_k\}}\left(1-\frac{1}{2\sigma^2}\left(\frac{\gamma}{a_1}+\frac{\gamma}{a_2}+\cdots+\frac{\gamma}{a_k}\right)\right)} \label{eq:6}\\
    &= \frac{{n-2 \choose k-1}}{{n-2 \choose k-1}-\frac{1}{2\sigma^2}\sum_{\{a_1, \cdots, a_k\}}\left(\frac{\gamma}{a_1}+\frac{\gamma}{a_2}+\cdots+\frac{\gamma}{a_k}\right)} \nonumber\\
    &= \frac{1}{1-\frac{\frac{1}{2 \sigma^2}\sum_{\{a_1, \cdots, a_k\}}\left(\frac{\gamma}{a_1}+\frac{\gamma}{a_2}+\cdots+\frac{\gamma}{a_k}\right)}{{n-2 \choose k-1}}} \nonumber,
\end{align}
\eqref{eq:6} holds because for any real number $x$, $e^{-x} \geq 1-x$. \eqref{eq:6} also requires $\sum_{\{a_1, \cdots, a_k\}}\left(1-\frac{1}{2\sigma^2}\left(\frac{\gamma}{a_1}+\frac{\gamma}{a_2}+\cdots+\frac{\gamma}{a_k}\right)\right) \\ \geq 0$, which we will enforce again later.

We next consider to upper-bound $\frac{\frac{\gamma}{2\sigma^2} \sum_{\{a_1, \cdots, a_k\}}\left(\frac{1}{a_1}+\frac{1}{a_2}+\cdots+\frac{1}{a_k}\right)}{{n-2 \choose k-1}}$.

If we place $j$ items into one specific bin, there are ${n-2-j \choose k-2}$ ways to put the remaining $n-1-j$ items to $k-1$ bins. Therefore, in the sum $\sum_{\{a_1, \cdots, a_k\}}\left(\frac{1}{a_1}+\frac{1}{a_2}+\cdots+\frac{1}{a_k}\right)$, for each bin $i$, $a_i=j~(1 \leq j \leq n-k)$ appears ${n-2-j \choose k-2}$ times. And there are $k$ bins, so we have:
\begin{align*}
    \frac{\frac{\gamma}{2\sigma^2} \sum_{\{a_1, \cdots, a_k\}}\left(\frac{1}{a_1}+\frac{1}{a_2}+\cdots+\frac{1}{a_k}\right)}{{n-2 \choose k-1}} &= \frac{\frac{\gamma}{2\sigma^2}k \left(1{n-3 \choose k-2}+\frac{1}{2}{n-4 \choose k-2} + \cdots + \frac{1}{n-k}{k-2 \choose k-2}\right)}{{n-2 \choose k-1}} \\
    &= \frac{\frac{\gamma}{2\sigma^2}k \sum_{i=1}^{n-k}\frac{1}{i}{n-2-i \choose k-2}}{{n-2 \choose k-1}}.
\end{align*}

By expanding and rearranging the terms, we get:
\begin{align}
\frac{\frac{\gamma}{2\sigma^2}k \sum_{i=1}^{n-k}\frac{1}{i}{n-2-i \choose k-2}}{{n-2 \choose k-1}} \nonumber &= \frac{\frac{\gamma}{2\sigma^2}k(k-1)}{(n-2)(n-3)\cdot \cdots (n-k)} \sum_{i=1}^{n-k}\frac{1}{i} \left(n-i-2\right)\left(n-i-3\right)\cdots \left(n-i-(k-1)\right)\nonumber \\
&= \frac{\frac{\gamma}{2\sigma^2}k(k-1)}{n-2} \sum_{i=1}^{n-k}\frac{1}{i}\frac{(n-i-2)(n-i-3)\cdots(n-i-(k-1))}{(n-3)(n-4)\cdot \cdots (n-k)} \\
&\leq \frac{\frac{\gamma}{2\sigma^2}k(k-1)}{n-2}\sum_{i=1}^{n-k}\frac{1}{i} \\
&\leq \frac{\frac{\gamma}{2\sigma^2}k(k-1)}{n-2}\left(1+\ln(n-k)\right) \label{eq:7}\\
&\leq  \frac{2\alpha^2 k(k-1)}{\sigma^2(n-2)}\left(1+\ln(n-k)\right) \label{eq:8}.
\end{align}

\eqref{eq:7} holds because $\sum_{i=1}^k \frac{1}{i} \leq 1 + \ln k~(\forall k)$, and \eqref{eq:8} is due to $\gamma \leq 4\alpha^2$ (with high probability).

Therefore, 
\begin{align*}
    \frac{1}{1-\frac{\frac{\gamma}{2\sigma^2}\sum_{\{a_1, \cdots, a_k\}}\left(\frac{1}{a_1}+\frac{1}{a_2}+\cdots+\frac{1}{a_k}\right)}{{n-2 \choose k-1}}} \leq  \frac{1}{1-\frac{2\alpha^2 k(k-1)}{\sigma^2(n-2)}(1+\ln(n-k))} \leq 1+\frac{\beta \alpha^2 k(k-1)}{\sigma^2(n-2)} \left(1+\ln(n-k)\right)~(\beta>0),
\end{align*}
where the second inequality holds when $\frac{\alpha^2 k(k-1)}{\sigma^2(n-2)}\left(1+\ln(n-k)\right) \leq \frac{1}{2}- \frac{1}{\beta}~(\beta>0)$.

Thus,
\begin{align*}
    \frac{P(b_1-c, b_2, \cdots, b_k | D_0)}{P(b_1-d, b_2, \cdots, b_k | D_0)} \leq 1+\frac{\beta \alpha^2 k(k-1)}{\sigma^2(n-2)} (1+\ln(n-k)).
\end{align*}
Similarly, for any $i~(1\leq i \leq k)$, we have:
\begin{align*}
        \frac{P(b_1, \cdots, b_i-c, \cdots, b_k | D_0)}{P(b_1, \cdots, b_i-d, \cdots, b_k | D_0)} \leq 1+\frac{\beta \alpha^2 k(k-1)}{\sigma^2(n-2)} (1+\ln(n-k)),
\end{align*}
and 
\begin{align*}
    \frac{P(b_1, \cdots, b_i+c, \cdots, b_k | D_0)}{P(b_1, \cdots, b_i+d, \cdots, b_k | D_0)} \leq 1+\frac{\beta \alpha^2 k(k-1)}{\sigma^2 (n-2)} (1+\ln(n-k)).
\end{align*}
Thus, 
\begin{align*}
    \frac{P(b_1, b_2,\cdots, b_k | D)}{P(b_1, b_2, \cdots, b_k | D')} \leq 1+\frac{\beta \alpha^2 k(k-1)}{\sigma^2(n-2)} (1+\ln(n-k)),
\end{align*}
which indicates that the sketching algorithm $\mathcal{M}$ with input size $n$ using Count Sketch with $k$ bins ($k$ counters) and 1 hash function (1 array) achieves $\ln\left(1+\frac{\beta \alpha^2 k(k-1)}{\sigma^2(n-2)}(1+\ln(n-k))\right)$-differential privacy with a large probability, where $\beta$ is a positive constant satisfying $\frac{\alpha^2k(k-1)}{\sigma^2 (n-2)}(1+\ln(n-k)) \leq \frac{1}{2}-\frac{1}{\beta}$. Since the $t$ counter arrays are pairwise independent, it follows that Count Sketch with $k$ bins ($k$ counters) and $t$ arrays achieves $t \cdot \ln\left(1+\frac{\beta \alpha^2 k(k-1)}{\sigma^2(n-2)}(1+\ln(n-k))\right)$-differential privacy with a large probability, where $\beta$ is a positive constant satisfying $\frac{\alpha^2k(k-1)}{\sigma^2(n-2)}(1+\ln(n-k)) \leq \frac{1}{2}-\frac{1}{\beta}$.
\end{proof}

\section{Proof for Theorem~\ref{theorem:fl_convergence}}\label{app:convergence}

{\bf Theorem}: (Convergence of \name in distributed SGD)
Assume that $E\left[\|g_k\|^2\right] \leq G^2$ for any input stochastic gradient $g_k$ on worker $k$ with dimension $n$, $E\left[\|x_i - x^*\|^2\right] \leq D^2$ at any iteration $i$, and the local model at device $k$ $f_k(x)$ is convex. Choose the step-size at round $i$  $\eta_i=\frac{c}{\sqrt{i}}$ where $c$ is a pre-defined positive number. Using a Count Sketch with $t$ hash functions and $k$ bins, with probability $1-\delta$, we have the following convergence rate on the global objective $F$:
\begin{equation*}
    F(\bar{x}_i)-F(x^*) \leq \frac{\frac{D^2}{2c}+c\sqrt{\frac{i+1}{i}}\left(n\mu^2+1\right)G^2}{\sqrt{i}},
\end{equation*}
where $x^*$ is the optimal solution to \eqref{eq:obj}, $k=O\left(\frac{e}{\mu^2}\right)$, $t=O\left(\ln(\frac{1}{\delta})\right)$, and $\bar{x}_i$ = $\frac{1}{i} \sum_{j=1}^i x_j$.

\begin{proof}

Because the local loss function $f_k$ and the global loss function $F$ are convex, we have
\begin{align*}
    \left\langle \nabla f_k(x_i), x_i-x^* \right\rangle \geq f_k(x_i)-f_k(x^*),
\end{align*}
and
\begin{align*}
    \langle \nabla F(x_i), x_i-x^* \rangle \geq F(x_i)-F(x^*).
\end{align*}
We use $\tilde{g}_k$ to denote the estimated stochastic gradient on worker $k$ after querying the sketch table at the current iteration $i$ (we omit the subscript $i$ for cleaner notations). Suppose $K$ worker are selected at each updating round, the updating rule is $x_{i+1} = x_{i} - \eta_i \frac{1}{K} \sum_{k=1}^K \tilde{g}_k$, and we have
\begin{align*}
    \mathbb{E}\left[\left\|x_{i+1}-x^*\right\|^2\right] &= \mathbb{E}\left[\left\|x_i-\eta_i \frac{1}{K} \sum_{k=1}^K \tilde{g}_k-x^*\right\|^2\right] \\
    &= \mathbb{E}\left[\left\|x_i-x^*\right\|^2\right] - 2\eta_i \mathbb{E}\left[\left\langle \frac{1}{K}\sum_{k=1}^K \tilde{g}_k, x_i-x^*\right\rangle\right] + \eta_i^2\mathbb{E}\left[\left\| \frac{1}{K}\sum_{k=1}^K \tilde{g}_k\right\|^2\right].
\end{align*}
From Lemma \ref{lemma:cs_error}, we know that the estimation error of Count Sketch is bounded. In particular, it holds that with probability $p  \geq 1-\delta$,
\begin{equation*}
    \mathbb{E}\left[\left\|{\tilde{g}_k-g_k}\right\|^2\right] \leq \mathbb{E}\left[n\left(\mu \left\|g_k\right\|_2\right)^2\right] \leq n\mu^2G^2,
\end{equation*}
where $k=O\left(\frac{e}{\mu^2}\right)$ and $t=O\left(\ln\left(\frac{1}{\delta}\right)\right)$.
Thus,
\begin{align*}
    \mathbb{E}\left[\left\|\tilde{g}_k\right\|^2\right] = \mathbb{E}\left[\left\|\tilde{g}_k-g_k+g_k\right\|^2\right]&\leq \mathbb{E}\left[2\left(\|\tilde{g}_k-g_k\|^2 + \|g_k\|^2\right)\right] \\
    &\leq 2\left(n\mu^2+1\right)G^2.
\end{align*}
Further, we have
\begin{align*}
    \mathbb{E}\left[\left\| \frac{1}{K}\sum_{k=1}^K \tilde{g}_k\right\|^2\right] &= \frac{1}{K^2}\mathbb{E}\left[\left\|\sum_{k=1}^K \tilde{g}_k\right\|^2\right] \\
    & \leq \frac{1}{K^2} \cdot K \cdot \sum_{k=1}^K \mathbb{E}\left[\left\|\tilde{g}_k\right\|^2\right] \\
    & \leq 2\left(n\mu^2+1\right)G^2.
\end{align*}
In addition, as sketches produce an unbiased estimation of the stochastic gradient, which is an unbiased estimation of the true gradient, we have $\mathbb{E}\left[\tilde{g}_k\right]=\mathbb{E}\left[\nabla f_k\right]$ (the expectation is with respect to the randomly sampled data point and the randomized sketch algorithm). And $\mathbb{E}\left[\frac{1}{K}\sum_{k=1}^K\tilde{g}_k\right] = \mathbb{E}\left[\frac{1}{K}\sum_{k=1}^K \nabla f_k\right] = \mathbb{E}\left[\nabla F(x_i)\right]$
Applying the bounded variance and the unbiased compression, it follows:
\begin{align*}
    \mathbb{E}\left[\|x_{i+1}-x^*\|^2\right] &\leq  \mathbb{E}\left[\|x_i-x^*\|^2\right] -2\eta_i \mathbb{E}\left[\langle \nabla F(x_i), x_i-x^*\rangle\right]+ 2\eta_i^2 \left(n\mu^2+1\right)G^2 \\ 
    & \leq  \mathbb{E}\left[\|x_i-x^*\|^2\right] -2\eta_i \left(F(x_i)-F(x^*)\right)+ 2\eta_i^2 \left(n\mu^2+1\right)G^2 \\
    \Rightarrow & 2\left(F(x_i)-F(x^*)\right) \leq \frac{1}{\eta_i}\mathbb{E}\left[\|x_i-x^*\|^2\right] - \frac{1}{\eta_i}\mathbb{E}\left[\|x_{i+1}-x^*\|^2\right] + 2\eta_i \left(n\mu^2+1\right)G^2. 
\end{align*}
Summarizing the inequalities for $j=1, \cdots, i$, we get
\begin{align*}
    2\sum_{j=i}^i \left(F(x_j)-F(x^*)\right) &\leq \frac{1}{\eta_j}\mathbb{E}\left[\|x_1-x^*\|^2\right] + \sum_{j=2}^i \left(\frac{1}{\eta_j}-\frac{1}{\eta_{j-1}}\right)\mathbb{E}\left[\|x_j-x^*\|^2\right] + \sum_{j=1}^k 2\eta_j \left(n\mu^2+1\right)G^2 \\
    &\leq \frac{1}{\eta_j}D^2 + \sum_{j=2}^i \left(\frac{1}{\eta_j}-\frac{1}{\eta_{j-1}}\right)D^2 + \sum_{j=1}^i 2\eta_j\left(n\mu^2+1\right)G^2\\
    & \leq \frac{\sqrt{i}}{c}D^2 + 2c\sqrt{i+1}\left(n\mu^2+1\right)G^2.
\end{align*}
From Jensen's inequality, we have
\begin{align*}
    \sum_{j=1}^i F(x_j)-F(x^*) \geq F(\bar{x}_i)-F(x^*),
\end{align*}
which indicates that
\begin{align*}
    F(\bar{x}_i)-F(x^*) \leq \frac{\frac{\sqrt{i}}{2c}D^2 + c\sqrt{i+1}\left(n\mu^2+1\right)G^2}{i}.
\end{align*}
This completes the proof.
\end{proof}

\newpage
\section{Evaluation Details}\label{app:eval}

\subsection{Implementation Details}\label{app:eval:implementaion}
We first describe the implementation details for each experiment comparing with different baselines.

\underline{Figure~\ref{fig:mnist}.} We perform error correction for both \name and the baseline method of adding Laplacian noise after sketching. At each round, each local worker computes the gap between the latest local gradient and the (estimated) aggregated gradients. Since data are identically distributed across all workers, if the estimation is correct, we expect that the gap should be small. Therefore, each worker sets half of the queried results to zero if he sees a large gap. For \name, we also generate a noise vector drawn from the same Gaussian as the raw gradients, and append that noise vector to the input for compression. The noise vector increases the size of the input, thus boosting the privacy for sketches. 

\underline{Figure~\ref{fig:shakespeare}.} We do not perform error correction for any methods, since the local gradients would be stale in federated learning due to device sampling. Similarly, we append a noise vector after the raw model updates.

\underline{Figure~\ref{fig:compare_cpsgd}.} We use the $\delta, s, p$ values suggested in~\citet{agarwal2018cpsgd}. We set up 6,000 workers (MNIST*) for this experiment, and select 10 at each round. We use a different variant of MNIST because the privacy bound of \cpsgd will be tighter with a larger number of total workers. For \cpsgd, we calculate the $k,m$ values (corresponding to the $k,m$ parameters in the original paper) based on the given $\varepsilon$ and compression ratios, as summarized in the following table~\ref{table: cpsgd}.

\begin{table}[h]
	\begin{center}
		\caption{Parameter configuratioins for \cpsgd}
		\label{table: cpsgd}
		\begin{tabular}{ c c c c } 
			\toprule
			$\varepsilon$ & Bits (Compression) & $k$ &  $m$ \\
			\hline
			5  & 4 (10$\times$) & 12 & 4 \\
			5  & 5 (6$\times$) & 23 & 9 \\
			10  & 4 (10$\times$) & 14 & 2 \\
			10  & 5 (6$\times$) & 28 & 4 \\
			\bottomrule
		\end{tabular}
	\end{center}
\end{table}

{\bf Datasets.} We summarize the statistics of the datasets below.

\setlength{\tabcolsep}{2.5pt}
\begin{table}[h]
	\begin{center}
		\caption{Statistics of Datasets}
		\label{table: data}
		\begin{tabular}{ lllll } 
			\toprule
			\textbf{Dataset} & \textbf{Workers} & \textbf{Param.} & 
			\multicolumn{2}{l}{\textbf{Samples/device}} \\
			\cmidrule(l){4-5}
			& &  & mean & stdev \\
			\hline
			MNIST  & 10 & 7,850 & 200 & 0 \\
			Shakespeare  & 46 & 39,720 & 742 & 548 \\
			MNIST*  & 6,000 & 7,850 & 10 & 0 \\
			\bottomrule
		\end{tabular}
	\end{center}
\end{table}

{\bf Hyper-parameters.} We assume that the input gradients are bounded by a constant $\alpha$ with a 90\% probability; therefore, we dynamically choose the 90th percentile value of the local gradient vector as $\alpha$ on each worker. For all experiments, we use a batch size of 10. The learning rates of MNIST, MNSIT*, and Shakespeare are 0.01, 0.01, and 0.8. For each comparison, we fix the mini-batch orders and (if needed) the selected devices per round.

\newpage
\subsection{Real Gradient Distributions}\label{app:eval:grad}

We plot the real gradient distributions for MNIST and model updates distribution for Shakespeare without compression in Figure~\ref{fig:mnist_grad} and Figure~\ref{fig:shakespeare_grad}. When applying \name, we observe that their gradient/model update distributions throughout the training across are similar. In MNIST, there are ten workers participating in training at each round and the data on the ten workers are identically distributed. We randomly select one worker at each round and report the gradient distribution from that worker. We repeat this for ten rounds. In Shakespeare, we only sample one device per round; therefore we show the model update distributions of ten selected devices in ten rounds. 

\begin{figure}[h]
    \centering
    \includegraphics[width=0.95\textwidth]{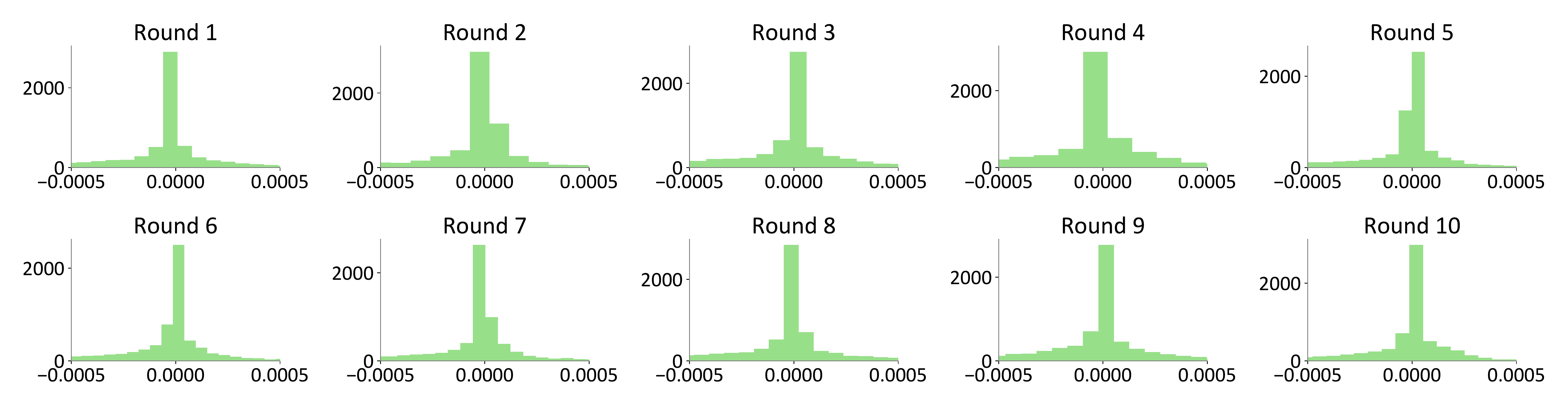}
    \caption{Gradient distributions of MNIST}
    \label{fig:mnist_grad}
\end{figure}
\begin{figure}[h]
    \centering
    \includegraphics[width=0.95\textwidth]{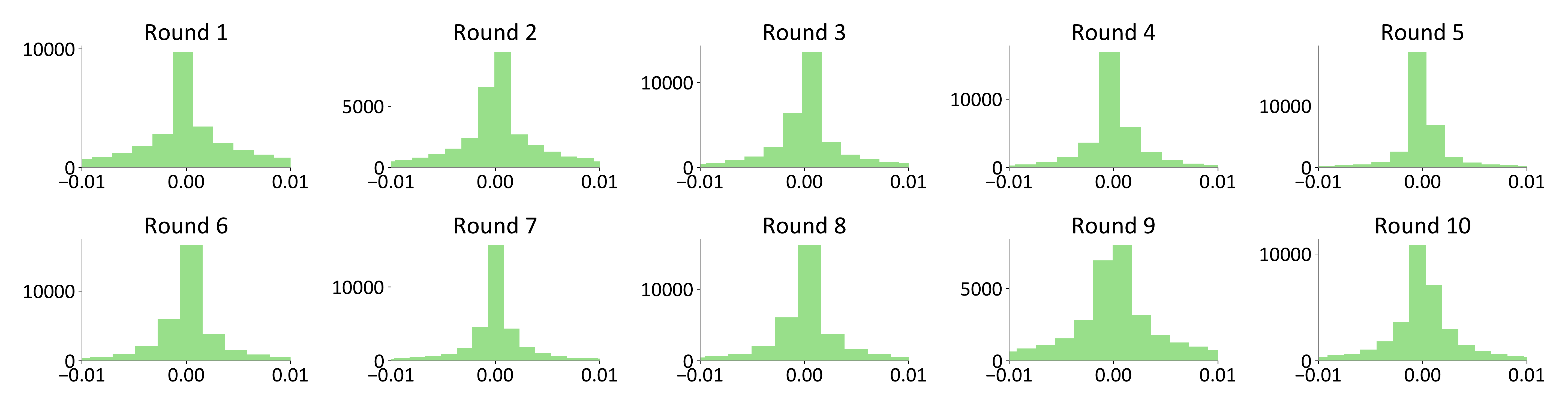}
    \caption{Distributions of model updates of Shakespeare}
    \vspace{-5mm}
    \label{fig:shakespeare_grad}
\end{figure}

\subsection{Privacy Parameter ($\varepsilon$) Distributions Across All Rounds}
\vspace{-3mm}
As mentioned before, due to the inherent privacy properties of sketches, a small amount of (Laplacian) noise is sufficient to provide a certain level of differential privacy guarantees. In Figure~\ref{fig:eps}, we visualize the histogram of the local $\varepsilon$ values on the MNIST dataset (with a $75\times$ compression ratio and $\varepsilon=1$) across all rounds.
We see that most of the $\varepsilon$ values guaranteed by sketches are bounded by the input privacy requirement 0.1, and we need to add additional Laplacian noise in a few cases to obtain a consistent privacy bound.

\begin{figure}[h]
    \centering
    \includegraphics[width=0.2\textwidth]{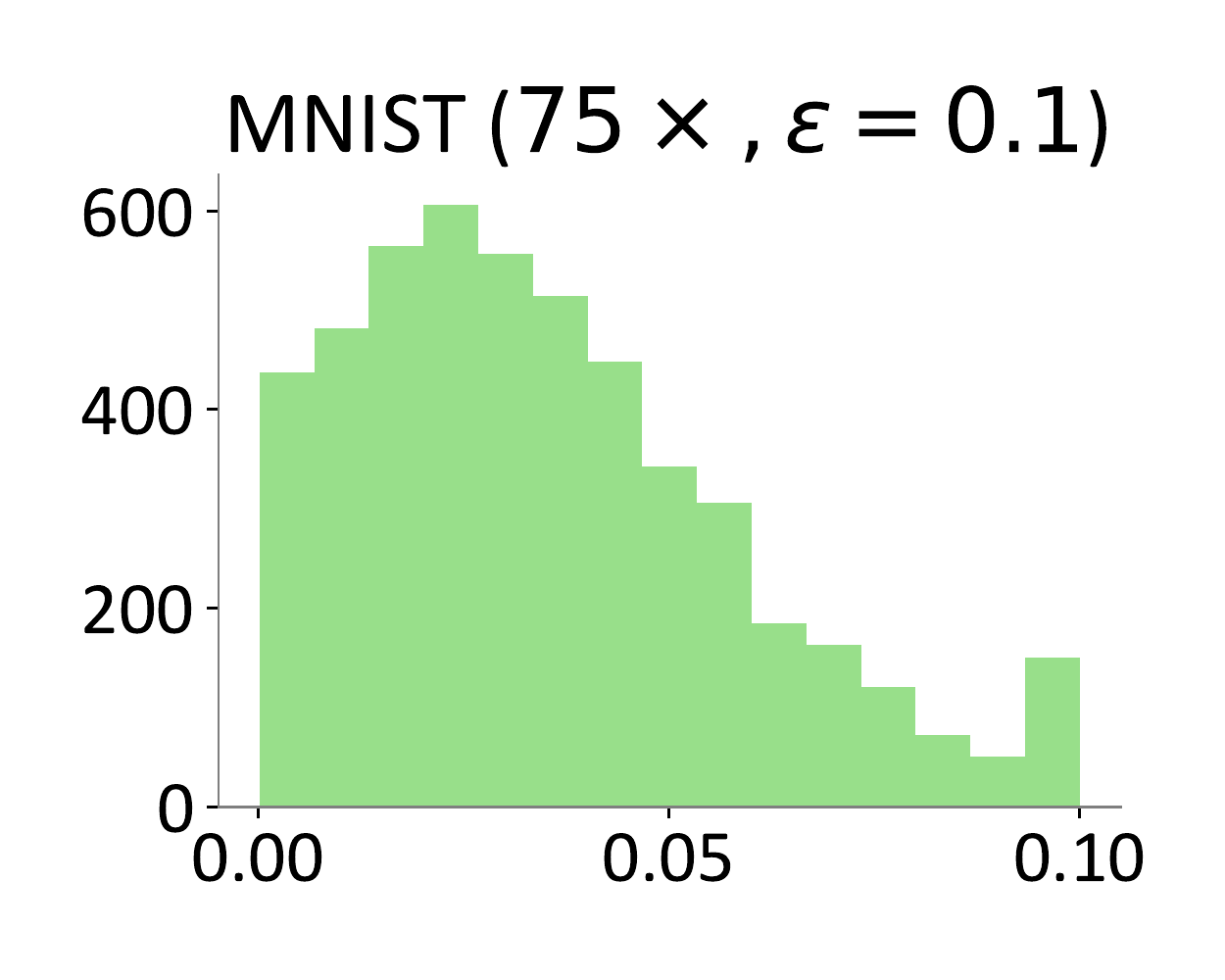}
    \caption{Local $\varepsilon$ values across all rounds when the input $\varepsilon$ requirement is 0.1.}
    \vspace{-3mm}
    \label{fig:eps}
\end{figure}

\end{document}